\documentclass{article}

\usepackage{microtype}
\usepackage{graphicx}
\usepackage{subcaption}
\usepackage{booktabs} 
\usepackage{hyperref}
\usepackage{wrapfig}

\usepackage{algorithm}
\usepackage{algorithmic}

\usepackage[final]{neurips_2025}

\usepackage{amsmath}
\usepackage{amssymb}
\usepackage{mathtools}
\usepackage{amsthm}
\usepackage{xcolor}
\usepackage{adjustbox}
\usepackage{subcaption}

\usepackage[capitalize,noabbrev]{cleveref}

\theoremstyle{plain}
\newtheorem{theorem}{Theorem}[section]
\newtheorem{proposition}[theorem]{Proposition}

\theoremstyle{definition}

\theoremstyle{remark}

\newcommand{\greent}[1]{\textcolor{black}{#1}}

\title{A Markov Decision Process for Variable Selection in Branch \& Bound}

\author{%
  Paul Strang\\
  EDF R\&D, France \\
  CNAM Paris, France \\
  \texttt{paul.strang@edf.fr} \\
  \And 
  Zacharie Alès \\
  ENSTA IP Paris, France \\
  CNAM Paris, France \\
  \texttt{zacharie.ales@ensta.fr} \\
  \And 
  Côme Bissuel
  \\
  EDF R\&D, France\\
  \texttt{come.bissuel@edf.fr} \\
  \And
  Olivier Juan
  \\
  EDF R\&D, France\\
  \texttt{olivier.juan@edf.fr} \\
  \And 
  Safia Kedad-Sidhoum
  \\
  CNAM Paris, France\\
  \texttt{safia.kedad\_sidhoum@cnam.fr} \\
  \And
  Emmanuel Rachelson
  \\
  ISAE-SUPAERO, France\\
  \texttt{emmanuel.rachelson@isae-supaero.fr} \\
}

\begin{document}

\maketitle

\begin{abstract}
    Mixed-Integer Linear Programming (MILP) is a powerful framework used to address a wide range of NP-hard combinatorial optimization problems, often solved by Branch and bound (B\&B). A key factor influencing the performance of B\&B solvers is the variable selection heuristic governing branching decisions. Recent contributions have sought to adapt reinforcement learning (RL) algorithms to the B\&B setting to learn optimal branching policies, through Markov Decision Processes (MDP) inspired formulations, and ad hoc convergence theorems and algorithms. In this work, we introduce BBMDP, a principled vanilla MDP formulation for variable selection in B\&B, allowing to leverage a broad range of RL algorithms for the purpose of learning optimal B\&B heuristics. Computational experiments validate our model empirically, as our branching agent outperforms prior state-of-the-art RL agents on four standard MILP benchmarks.
\end{abstract}

\section{Introduction}
\label{sec:intro}
    Mixed-Integer Linear Programming (MILP) is a subfield of combinatorial optimization (CO), a discipline that aims at finding solutions to optimization problems with large but finite sets of feasible solutions. Specifically, mixed-integer linear programming addresses CO problems that are NP-hard, meaning that no polynomial-time resolution algorithm has yet been discovered to solve them. Mixed-integer linear programs are used to solve efficiently a vast range of high-dimensional combinatorial problems, spanning from operations research \citep{hillier2015introduction} to the fields of deep learning \citep{tjengevaluating}, finance \citep{mansini2015linear}, computational biology \citep{gusfield2019integer}, and fundamental physics \citep{barahona1982computational}. 
    MILPs are traditionally solved using Branch and bound (B\&B) \citep{land1960automatic}, an algorithm which methodically explores the space of solutions by dividing the original problem into smaller sub-problems, while ensuring the optimality of the final returned solution. 
    Intensively developed since the 1980s \citep{bixby_brief_2012}, MILP solvers based on the B\&B algorithm are high-performing tools. 
    In particular, they rely on complex heuristics fine-tuned by experts on large heterogeneous benchmarks \citep{miplib_2021}. 
    Hence, in the context of real-world applications, in which similar instances with slightly varying inputs are solved on a regular basis, there is a huge incentive to reduce B\&B total solving time by learning efficient tailor-made heuristics. The branching heuristic, or variable selection heuristic, which determines how to iteratively partition the space of solutions, has been found to be critical to B\&B computational performance \citep{achterberg2013mixed}. Over the last decade, many contributions have sought to harness the predictive power of machine learning (ML) to learn better-performing B\&B heuristics \citep{bengio_machine_2021, scavuzzo_machine_2024}.
    By using imitation learning (IL) to replicate the behaviour of a greedy branching expert at lower computational cost, \citet{gasse_exact_2019} established a landmark result as they first managed to outperform a solver relying on human-expert heuristics. 
    Building on the works of \citet{gasse_exact_2019} and \citet{he_learning_2014}, who proposed a Markov decision process (MDP) formulation for node selection in B\&B, several contributions succeeded in learning efficient branching strategies by reinforcement \citep{etheve_reinforcement_2020, scavuzzo_learning_2022, parsonson_reinforcement_2022}, without surpassing the performance achieved by the IL approach. 
    Yet, if the performance of IL heuristics are caped by that of the suboptimal branching experts they learn from, the performance of RL branching strategies are, in theory, only bounded by the maximum score achievable. 
    We note that in order to cope with dire credit assignment problems \citep{pignatelli2023survey} induced by the sparse reward model described in \citet{he_learning_2014}, prior research efforts have shifted away from the traditional Markov decision process framework, finding it impractical for learning efficient branching strategies. Instead, \citet{etheve_reinforcement_2020}, \citet{scavuzzo_learning_2022} and \citet{parsonson_reinforcement_2022} have adopted unconventional MDP-inspired formulations to model variable selection in B\&B. \\


    \setlength{\intextsep}{0pt}
    \setlength{\columnsep}{10pt}
    \begin{wrapfigure}{r}{0.40\textwidth}
        \includegraphics[trim={0pt 0pt 0pt 0pt}, clip, width=0.38\textwidth]{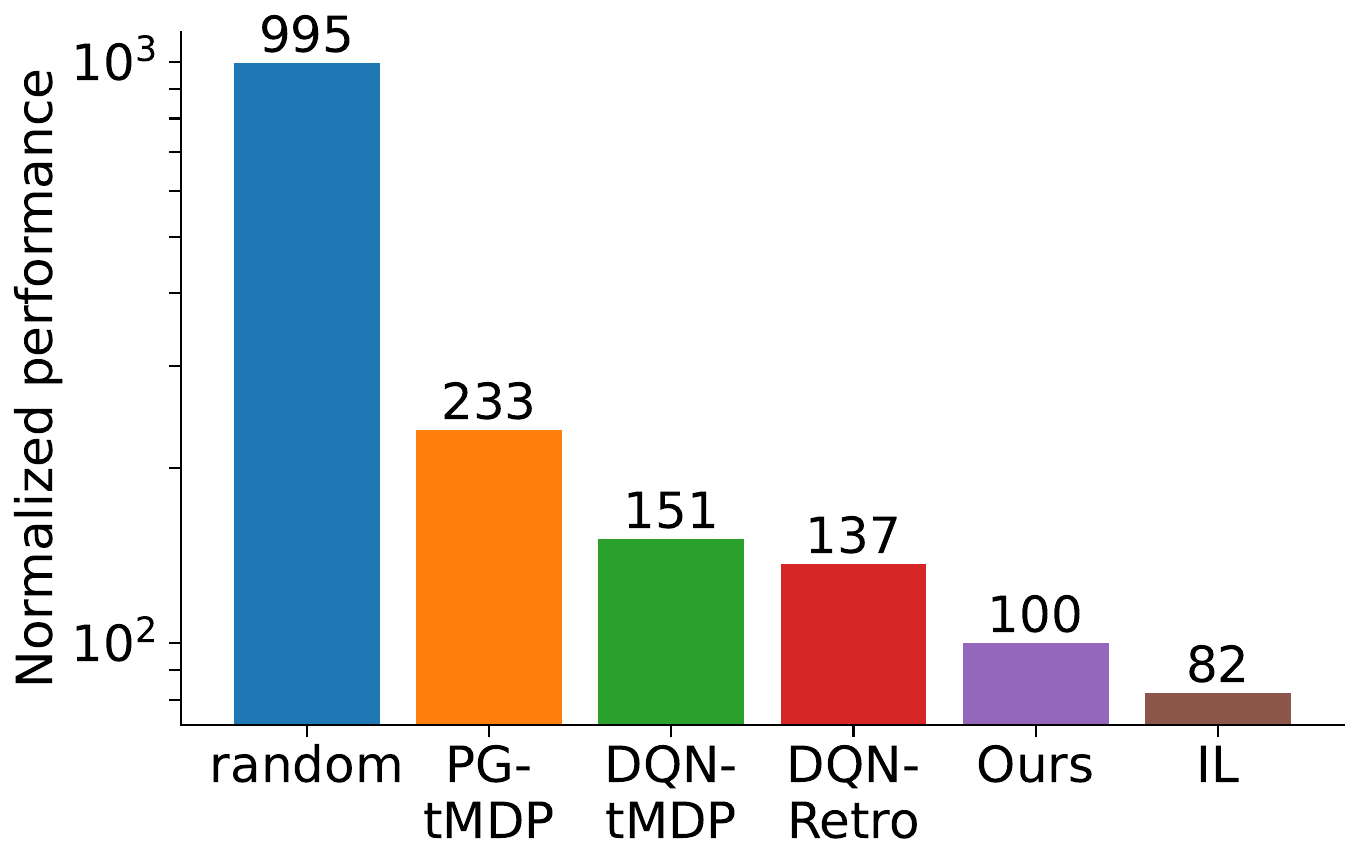}
      \caption{Normalized scores in log scale of IL, RL and random agents across the Ecole benchmark \citet{prouvost_ecole_2020}.}
      \label{fig:teasing}
    \end{wrapfigure}
    
    In this work, we show that, despite improving the empirical performance of RL algorithms, these alternative formulations introduce approximations which undermine the asymptotic performance of RL branching agents in the general case. 
    In order to address this issue, we introduce branch and bound Markov decision processes (BBMDP), a principled vanilla MDP formulation for variable selection in B\&B, which preserves convergence properties brought by previous contributions without sacrificing optimality. 
    Our new formulation allows to define a proper Bellman optimality operator, which, in turn, enables to unlock the full potential of state-of-the-art approximate dynamic programming algorithms \citep{hessel_rainbow_2017, dabney2018implicit, farebrother_stop_2024} for the purpose of learning optimal B\&B branching strategies. 
    We evaluate our method on four classic MILP benchmarks, achieving state-of-the art performance and dominating previous RL agents while narrowing the gap with the IL approach of \citet{gasse_exact_2019}, as shown in Figure \ref{fig:teasing}.

\section{Problem statement} 
    
    \subsection{Mixed-integer linear programming}
    \label{sec:b&b}
    We consider mixed-integer linear programs (MILPs), defined as: 
        \[ P : 
            \left\{
                \begin{array}{ll}
                    \min c^\top x \\
                    l \leq x \leq u \\
                     Ax \leq b \; ; \; x \in \mathbb{Z}^{|\mathcal{I}|} \times \mathbb{R}^{n-|\mathcal{I}|} \\

                \end{array}
            \right.
        \]
    with~$n$ the number of variables,~$m$ the number of linear constraints,~$l, u \in \mathbb{R}^n$ the lower and upper bound vectors,~$A\in \mathbb{R}^{m\times n}$ the constraint matrix,~$b\in \mathbb{R}^m$ the right-hand side vector,~$c\in \mathbb{R}^n$ the objective function, and~$\mathcal{I}$ the indices of integer variables. Throughout this document, we are interested in repeated MILPs of fixed dimension~$\{P_{i} = (A_i, b_i, c_i, l_i, u_i)\}_{i\in N}$ sampled according to an unknown distribution~$p_0 : \Omega \rightarrow \mathbb{R}^{m\times n} \times \mathbb{R}^{m} \times \mathbb{R}^{n} \times \mathbb{R}^{n} \times \mathbb{R}^{n}$. 
    
    In order to solve MILPs efficiently, the B\&B algorithm iteratively builds a binary tree $(\mathcal{V}, \mathcal{E}) $ where each node corresponds to a MILP, starting from the root node $v_0 \in \mathcal{V}$ representing the original problem $P_0$. 
    The incumbent solution $\bar{x} \in \mathbb{Z}^{|\mathcal{I}|} \times \mathbb{R}^{n-|\mathcal{I}|}$ denotes the best feasible solution found at current iteration, its associated value $GU\!B = c^\top \bar{x}$ is called the \textit{global upper bound} on the optimal value. 
    The overall state of the optimization process is thus captured by the triplet $s = (\mathcal{V}, \mathcal{E}, \bar{x})$, we note $\mathcal{S}$ the set of all such triplets.\footnote{To account for early resolution steps where no incumbent solution has yet been found, we define a special value for $\bar{x}$, whose $GU\!B=\infty$. For the sake of simplicity, we make this implicit in the remainder of the paper.} 
    Throughout the optimization process, B\&B nodes are explored sequentially. We note $\mathcal{C}$ the set of visited or closed nodes, and $\mathcal{O}$ the set of unvisited or open nodes, such that $\mathcal{V} = \mathcal{C} \cup \mathcal{O}$. Originally, $\mathcal{O} = \{ v_0\}$ and $\mathcal{C} = \emptyset$.
    At each iteration, the node selection policy $\rho : \mathcal{S} \rightarrow \mathcal{O}$ selects the next node to explore. Since $\rho$ necessarily defines a total order on nodes $o_i \in \mathcal{O}$, we can arrange indices such that $o_1 = \rho(s)$ denotes $v_t$ the node currently explored at step $t$. 
    
    Figure~\ref{fig:BB} illustrates how B\&B operates on an example.
    At each iteration, let $x^*_{LP} \in \mathbb{R}^n$ be the optimal solution to the linear relaxation of~$P_{t}$, the problem associated with $v_t$. 
    If $P_t$ admits no solution, 
    $v_t$ is marked as visited and the branch is pruned by infeasibility. If $x^*_{LP} \in \mathbb{R}^n$ exists, and $GU\!B < c^\top x^*_{LP}$, no integer solution in the subsequent branch can improve $GU\!B$, thus $v_t$ is marked as closed and the branch is pruned by bound.
    If $x^*_{LP}$ is not dominated by $\bar{x}$ and $x^*_{LP}$ is feasible (all integer variables in $x^*_{LP} \in \mathbb{R}^n$ have integer values), a new incumbent solution $\bar{x} = x^*_{LP}$ has been found.
    Hence $GU\!B$ is updated and $v_t$ is marked as visited while the branch is pruned by integrity.
    Else, $x^*_{LP}$ admits fractional values for some integer variables. The branching heuristic $\pi: \mathcal{S} \rightarrow \mathcal{I}$ selects a variable $x_b$ with fractional value $\hat{x}_b$, to partition the solution space. As a result, two child nodes $(v_{-}, v_{+})$, with associated MILPs~$P_{-} = P_{t} \, \cup \, \{ x_b \leq \lfloor \hat{x}_b \rfloor \}$ and~$P_{_{+}} = P_{t} \, \cup \, \{ x_b \geq \lceil \hat{x}_b \rceil\}$, are added to the current node.\footnote{$\hat{x}_b$ denotes the value of $x_b$ in $x^*_{LP}$. We use the symbol $\cup$ to denote the refinement of the bound on $x_b$ in $P_{t}$.} Their linear relaxation is solved, before they are added to the set of open nodes $\mathcal{O}$ and $v_t$ is marked as visited.

\begin{figure*}[t]
    \centering
    \includegraphics[width=\textwidth]{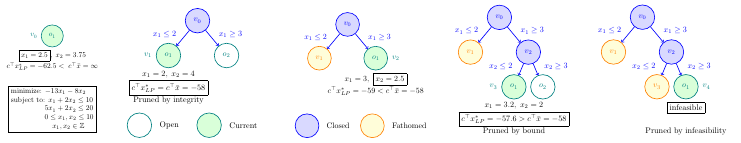}
    \caption{Solving a MILP by B\&B using variable selection policy $\pi$ and node selection policy $\rho$. Each node $v_i$ represents a MILP derived from the original problem, each edge represents the bound adjustment applied to derive child nodes from their parent. At each step, nodes $o_i \in \mathcal{O}$ are re-indexed according to $\rho$.}
    \label{fig:BB}
\end{figure*}

    This process is repeated until $\mathcal{O} = \emptyset$ and $\bar{x}$ is returned. The dynamics of the B\&B algorithm between two branching decisions can be described by the function $\kappa_{\rho}: \mathcal{S} \times \mathcal{I} \rightarrow \mathcal{S}$, such that $s' =  \kappa_{\rho}(s, \pi(s))$. By design, B\&B does not terminate before finding an optimal solution and proving its optimality. Consequently, optimizing the performance of B\&B on a distribution of MILP instances is equivalent to minimizing the expected solving time of the algorithm. As~\citet{etheve_solving_2021} evidenced, the variable selection strategy~$\pi$ is by far the most critical B\&B heuristic in terms of computational performance. In practice, the total number of nodes of the B\&B tree is used as an alternative metric to evaluate the performance of branching heuristics $\pi$, as it is a hardware-independent proxy for computational efficiency. 
    Under these circumstances, given a fixed node selection strategy~$\rho$, the optimal branching strategy~$\pi^*$ associated with a distribution~$p_0$ of MILP instances can be defined as: 
    \begin{equation}
    \label{eq:obj}
        \pi^* = \arg \min_{\pi} \mathbb{E}_{P\sim p_0}(|BB_{(\pi, \rho)}(P)|),
    \end{equation}
    with~$|BB_{(\pi,\rho)}(P)|$ the size of the~$B\&B$ tree after solving $P$ to optimality following strategies ($\pi, \rho$).

    \subsection{Reinforcement learning}
    We consider the setting of discrete-time, deterministic MDPs \citep{puterman2014markov} defined by the tuple $(\mathcal{S}, \mathcal{A}, \mathcal{T}, p_0, \mathcal{R})$.
    At each time step $t$, the agent observes $s_t \in \mathcal{S}$ the current state of the environment, before executing action $a_t \in \mathcal{A}$, and receiving reward $r_t = \mathcal{R}(s_t, a_t)$. 
    The Markov transition function $\mathcal{T} : \mathcal{S} \times \mathcal{A} \rightarrow \mathcal{S}$ models the dynamics of the environment. 
    In particular, it satisfies the Markov property: conditionally to $s_t$ and $a_t$, $s_{t+1}$ is independent of all past visited states and actions. 
    Given a trajectory starting in state $s_0$ sampled according to the initial distribution $p_0$, the total gain is defined for all $t\geq 0$ as $G_t = \sum_{t'=t}^{\infty} \gamma^{t'-t} \mathcal{R}(s_{t'}, a_{t'})$, with $\gamma \in [0,1]$. 
    The objective of an RL agent is to maximize the expected gain of the trajectories yielded by its action selection policy $\pi: \mathcal{S} \rightarrow \mathcal{A}$. 
    This is equivalent to finding the policy maximizing value functions $V^{\pi}(s_t) = \mathop{\mathbb{E}}_{a_{t'} \sim \pi(s_{t'})} [G_{t}]$ and $Q^{\pi}(s_t, a_t) = \mathcal{R}(s_{t}, a_{t}) + \gamma V^{\pi}(s_{t+1})$. 
    The optimal Q-value function $Q^*$ indicates the highest achievable cumulated gain in the MDP. It satisfies the Bellman optimality equation $Q(s, a) = \mathcal{R}(s, a) + \gamma \cdot \max_{a' \in \mathcal{A}} Q(s', a')$, noting $s'= \mathcal{T}(s,a)$ for $(s,a) \in \mathcal{S} \times \mathcal{A}$. The optimal policy is retrieved by acting greedily according to the learned $Q$-value function: $\pi^*(s) = \arg \max_{a \in \mathcal{A}} Q^*(s,a)$. 

    \subsection{Related work}
    Following the seminal work by \citet{gasse_exact_2019}, few contributions have proposed to build more complex neural network architectures based on transformers \citep{lin_learning_2022} and recurrence mechanisms \citep{seyfi_exact_2023} to improve the performance of IL branching agents, with moderate success. In parallel, theoretical and computational analysis \citep{bestuzheva_scip_2021, sun_improving_2022} have shown that neural networks trained by imitation could not rival the tree size performance achieved by strong branching (SB), the branching expert used in \citet{gasse_exact_2019}. In fact, low tree sizes associated with SB turn out to be primarily due to the formulation improvements resulting from the massive number of LPs solved in SB, not to the intrinsic quality of the branching decisions themselves. 
    
    Since branching decisions are made sequentially, reinforcement learning appears as a natural candidate to learn good branching policies.
    \citet{etheve_reinforcement_2020} and \citet{scavuzzo_learning_2022} proposed the model of TreeMDP, in which state $s_i = (P_{i}, x^{*}_{LP, i}, \bar{x}_{i})$ consists in the MILP associated with node $v_i$ along with the solution of its linear relaxation and the incumbent solution at $v_i$.
    The actions available at $s_i$ are the set of fractional variables in $x^{*}_{LP, i}$.
    Given $(s_i, a_i)$ the tree Markov transition function produces two child node states $(s^-_{i}, s^+_{i})$ that can be visited in any order.
    Crucially, when the B\&B tree is explored in depth-first-search (DFS), TreeMDP trajectories can be divided in independent subtrees, allowing to learn policies minimizing the size of each subtree independently.
    This helps mitigate credit assignment issues that arise owing to the length of episode trajectories. 
    Subsequently, \citet{parsonson_reinforcement_2022} found the DFS node selection policy to be highly detrimental to the computational performance of RL branching strategies.
    Assuming that RL branching agents trained following advanced node selection strategies would perform better despite the lack of theoretical guarantee, they proposed to learn from retrospective trajectories, diving trajectories built from original TreeMDP episodes.
    In fact, \citet{parsonson_reinforcement_2022} found retrospective trajectories to alleviate the partial observability induced by the ``disordered'' exploration of the tree and outperform prior RL agents.
    
    A large body of work has proposed to learn, either by imitation or reinforcement, better-performing B\&B heuristics outside of variable selection \citep{ nair_solving_2021, paulus_learning_2022}. 
    RL contributions in primal search \citep{sonnerat_learning_2022, wu_deep_2023} node selection \citep{he_learning_2014, etheve_solving_2021} and cut selection \citep{tang_reinforcement_2020, song_general_2020, wang_learning_2023} have all relied on the TreeMDP framework to train their agents, simply adapting the action set to the task at hand. Finally, machine learning applications in combinatorial optimization are not limited to B\&B. For example, in the context of routing or scheduling problems where exact resolution rapidly becomes prohibitive, agents are trained to learn direct search heuristics \citep{kool_attention_2019, grinsztajn_population-based_2022, chalumeau2023combinatorial} yielding high-quality feasible solutions.

\section{Branch and bound Markov decision process}
    \label{sec:tempMDP}
    By using the current B\&B node as the observable state, 
    prior attempts to learn optimal branching strategies have relied on the TreeMDP formalism to train RL agents. 
    However, TreeMDPs are not MDPs, as they do not define a Markov process on the state random variable (for instance, a transition yields two states and is hence not a stochastic process on the state variables). 
    As a result, this forces \citet{etheve_reinforcement_2020} and \citet{scavuzzo_learning_2022} to redefine Bellman updates and derive \textit{ad hoc} convergence theorems for TD(0), value iteration, and policy gradient algorithms.
    In order to leverage broader theoretical results from the reinforcement learning literature, we propose a description of variable selection in B\&B as a proper Markov decision process. 
    
    \subsection{Definition}
    \label{sec:def}
    The problem of finding an optimal branching strategy according to Eq.~\eqref{eq:obj} can be described as a regular deterministic Markov decision process.
    To this end, we introduce Branch and bound Markov decision processes (BBMDP), taking $\gamma = 1$ since episodes horizons are bounded by the (finite) largest possible number of nodes:

    \textbf{State-action space.}  $\mathcal{S}$ is the set of all B\&B trees $s_t = (\mathcal{V}_t, \mathcal{E}_t, \bar{x}_t)$. 
    Note that this includes intermediate B\&B trees, whose incumbent solutions $\bar{x}_t$ are yet to be proven optimal. $\mathcal{A}$ is the set of all integer variables indices $\mathcal{I}$. 
    
    \textbf{Transition function.} The Markov transition function is defined as~$\mathcal{T} = \kappa_\rho$ with $\kappa_\rho$ the branching operation described in Section \ref{sec:b&b}. 
    Note that if the variable associated with $a_t$ is not fractional in $x^*_{LP}$, then $s_{t+1} = \mathcal{T}(s_t, a_t) = s_t$ as relaxing a variable that is not fractional has no impact on the LP relaxation.
    Importantly, all states for which $\mathcal{O}=\emptyset$ are terminal states.
    
    \textbf{Starting states.} Initial states are single node trees, where the root node is associated to a MILP $P_0$ drawn according to the distribution $p_0$ defined in Section~\ref{sec:b&b} (hence the use of $p_0$ for both the initial problem $P_0$ and the MDP's initial state $s_0$).
    
    \textbf{Reward model.}
    We define $\mathcal{R}(s,a)=-2$ for all transitions until episode termination. Since each transition adds two B\&B nodes, the overall value of a trajectory is the opposite of the number of nodes added to the B\&B tree from the root node, which aligns with the definition of Eq.~\eqref{eq:obj}.

     Unlike in TreeMDP, the current state is defined as the state of the entire B\&B tree, rather than merely the current B\&B node. The transition function returns a B\&B tree whose open nodes are sorted according to the node selection policy $\rho$, thus reflecting the true dynamics of the B\&B algorithm, instead of a couple of pseudo-states associated with the child nodes of the last node expansion. Note that the definition above sets BBMDPs among the specific class of MDPs called stochastic shortest path problems \citep{puterman2014markov}.

    \subsection{Misconceptions when learning to explore B\&B trees}
    \label{sec:core}
    
    Like in TreeMDP, episode trajectories can be decomposed in independent subtree trajectories, to facilitate RL agents training. Consider a deterministic branching policy $\pi$, and let us we rewrite $V^{\pi}$ and $Q^{\pi}$ to exhibit the tree structure. 
    Given an open node $o_i \in \mathcal{O}_t$, we note $T(o_i)$ the subtree rooted in $o_i$, and define $\bar{V}^{\pi}(s_t, o_i)$ the 
    function that returns the opposite of the size of the subtree rooted in $o_i$, when branching according to policy $\pi$ starting from state $s_t$ until episode termination. 
    Then $V^\pi$ can be expressed as: 
    \begin{equation}
        \label{eq:v=w}
        V^{\pi}(s_t) 
        = \sum_{o_i \in \mathcal{O}_t} \bar{V}^{\pi}(s_t, o_i).
    \end{equation}
    In plain words, the total number of nodes that will be added to the B\&B tree past $s_t$ is the sum of the sizes of all the subtrees $T(o_i)$ rooted in the open nodes of $s_t$. 
    Because the full B\&B tree is a complex object to manipulate, it is tempting to discard the tree structure in $s_t$ and define $\bar{V}$-value functions merely as functions of $(o_i,\bar{x}_{o_i})$, rather than functions of $(s_t, o_i)$, where $\bar{x}_{o_i} \in \mathbb{R}^n$ is the incumbent solution when $o_i$ is processed by the B\&B algorithm, at time step $\tau_i$. 
    The rationale for such value functions is that the size of the subtree rooted in $o_i \in \mathcal{O}_t$, for a given incumbent solution $\bar{x}_{o_i}$, should be the same, regardless of the parents of $o_i$, its position in the tree, or the branching decisions taken in subtrees $T(o_j)$ for $j \neq i$.
    \textbf{It turns out, this last statement does not always hold}, quite counter-intuitively. 
    Let us write $\tau_i$ the time steps at which the nodes $o_i\in\mathcal{O}_t$ are selected by the node selection strategy $\rho$.\footnote{Following our indexation of $o_i \in \mathcal{O}_t$, we have $t = \tau_1 < ... < \tau_i < ... < \tau_{|\mathcal{O}_t|}$.}
    Now, consider for instance a node selection procedure $\rho$ that performs a breadth-first search through the tree. 
    The number of nodes in $T(o_i)$ will depend strongly on whether an improved incumbent solution $\bar{x}_{o_i}$ was found in the subtrees explored between $s_t$ and $s_{{\tau}_i}$, and, in turn, on the branching decisions taken in these subtrees.
    This example highlights the major importance of the node selection strategy $\rho$, when one wishes to define subtree sizes based solely on $(o_i,\bar{x}_{o_i})$.
    
    Consider now two open nodes $o_i$ and $o_j$ in $\mathcal{O}_t$. 
    Conversely to the previous example, if one can guarantee that the subtree rooted in $o_j$ will be solved to optimality before $o_i$ is considered for expansion in the B\&B process, then the number of nodes in $T(o_i)$ will not be affected by the branching decisions taken at any node under $o_j$.
    In fact, if $o_j$ is solved to optimality, $\bar{x}_{o_i}$ will either not change if no feasible solution in $T(o_j)$ improves $GU\!B$, or either be the best feasible solution of the MILP associated with $o_j$, which does not depend on the series of actions taken in $T(o_j)$.
    In other words, to make sure that the size of $T(o_i)$ does only depend on the branching decisions taken in $T(o_i)$, all nodes $o_j \in \mathcal{O}_{t}$ must have been either fully explored or strictly unexplored at $\tau_i$.
    Applying this argument recursively induces that the only node selection strategy which enables predicting a subtree size only based on $(o_i,\bar{x}_{o_i})$, is a depth-first search (DFS) exploration of the B\&B tree. 
    Therefore, we consider $\rho=DFS$ and write $\bar{V}^\pi(o_i,\bar{x}_{o_i})$ the opposite of the size of $T(o_i)$ in this context.
    We can now derive a refined Bellman update to train branching agents in BBMDP. 
    \begin{proposition}
        \label{prop:BB_mdp}
        In DFS-BBMDP, the Bellman equation $V^\pi(s) = \mathcal{R}(s, \pi(s)) + V^\pi(s')$ yields: 
        \begin{equation}
            \label{eq:prop2}
            \bar{V}^{\pi}(o_1, \bar{x}_{o_1}) = -2 + \bar{V}^{\pi}(o'_1, \bar{x}_{o'_1}) + \bar{V}^{\pi}(o'_2, \bar{x}_{o'_2}),
        \end{equation}
        with $o_1=\rho(s)$, $o'_1 = \rho(\mathcal{T}(s,\pi(s))$, and $o'_2$ is the sibling of $o'_1$ in the B\&B tree.
    \end{proposition}
    \begin{proof}
        Eq.~\eqref{eq:prop2} follows directly from injecting Eq.~\eqref{eq:v=w} in the Bellman equation, and observing that most terms in the sums simplify as $\bar{V}^\pi(o_{i}, \bar{x}_{o_i}) = \bar{V}^\pi(o'_{i+1}, \bar{x}_{o'_{i+1}})$ for $i \geq 2$.
    \end{proof}
    Keeping the same notation convention for $o_1$, $o'_1$ and $o'_2$, we define $\bar{Q}^\pi(o_1, \bar{x}_{o_1}, a) = -2 + \bar{V}^{\pi}(o'_1, \bar{x}_{o'_1}) + \bar{V}^{\pi}(o'_2, \bar{x}_{o'_2})$. Analogous to $\bar{V}^\pi$, the $\bar{Q}^\pi$ function returns the opposite of the size of the subtree rooted in $o_i \in \mathcal{O}_t$ when branching on action $a \in \mathcal{A}$ at $o_i$ and following policy $\pi$ until $T(o_i)$ is fathomed.
    Note that if $\bar{V}^\pi$ and $\bar{Q}^\pi$ are not strictly value functions, they naturally emerge when applying Bellman equations to BBMDP value functions under $\rho = DFS$.
    We stress that, in depth-first search BBMDPs, it is not necessary to learn $Q^*$ to derive $\pi^*$, since acting according to a policy minimizing the size of the subtree rooted in the current B\&B node is equivalent to acting
    according to a global optimal policy:
    \begin{equation}
        \label{eq:pi_star}
        \pi^*(s) = \arg \max_{a \in \mathcal{A}} \, Q^*(s,\,a) = \arg \max_{a \in \mathcal{A}} \, \bar{Q}^*(o_1, \bar{x}_{o_1}, a).
    \end{equation}


    \begin{figure}[t]
         \centering
         \hspace*{-0.35cm}
         \begin{subfigure}[b]{0.49\textwidth}
             \centering
             \includegraphics[height=4em]{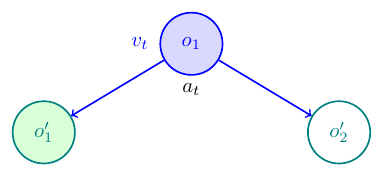}
             \caption{$1$-step BBMDP Bellman equation: \\ $\bar{V}^\pi(o_{1},\bar{x}_{o_1}) = -2 + \bar{V}^\pi(o'_{1},\bar{x}_{o'_1}) + \bar{V}^\pi(o'_2,\bar{x}_{o'_2})$}
             \label{fig:clean_one_step}
         \end{subfigure}
         \hspace*{1mm}
         \begin{subfigure}[b]{0.49\textwidth}
            \centering
             \includegraphics[height=4em]{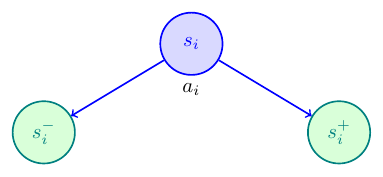}
             \caption{$1$-step TreeMDP Bellman equation:\\ $V^\pi(s_i) = -2 + V^\pi(s^-_i) + V^\pi(s^+_i)$}
             \label{fig:tree_one_step}
         \end{subfigure}
         \begin{subfigure}[b]{0.49\textwidth}
             \includegraphics[height=8em]{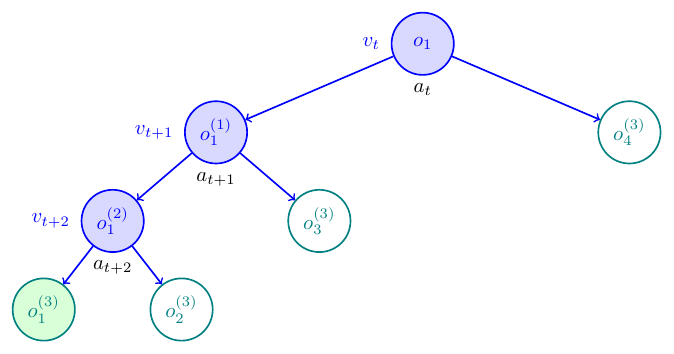}
             \caption{$k$-step BBMDP Bellman equation: \\ $\bar{V}^\pi(o_1,\bar{x}_{o_1})
             = -2k + \sum_{i=1}^{k+1}  \bar{V}^\pi(o^{(k)}_i,\bar{x}_{o^{(k)}_i})
             $}
             \label{fig:v_tree}
         \end{subfigure}
         \begin{subfigure}[b]{0.49\textwidth}
             \centering
             \includegraphics[height=8em]{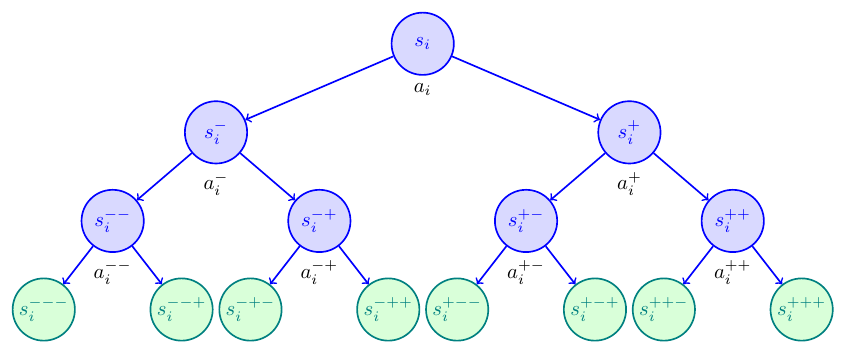}
             \caption{$k$-step TreeMDP Bellman equation: $\\ V^\pi(s_0) = -2^1 -2^2 - ... - 2^{k} + \sum_{i=1}^{2^k} V^\pi(s^j_i)$}
             \label{fig:w_tree}
         \end{subfigure}
        \caption{\greent{When minimizing 1-step temporal difference}, TreeMDP and BBMDP yield equivalent results, see \ref{fig:clean_one_step}, \ref{fig:tree_one_step}. However, over $k$-step trajectories, the two methods diverge as shown in \ref{fig:v_tree}, \ref{fig:w_tree}.}
        \label{fig:tree vs clean}
    \end{figure}

    \subsection{Approximate dynamic programming}
    \label{sec:bellman}
    The previous properties enable learning $\pi^*$ by training a neural network to approximate $\bar{Q}^*$ using traditional temporal difference (TD) algorithms. 
    Notably, $\bar{Q}^*$ is easier to learn than $Q^*$, as it relies solely on quantities observable at time $t$, whereas the previous decomposition of $Q^*$ depends on all $\bar{x}_{o_i}$ for $o_i \in \mathcal{O}_t$, which are not known at time $t$. 
    Moreover, $\bar{Q}^*$ trains on much shorter trajectories than $Q^*$, which helps mitigate credit assignement issues.\footnote{On average, the length of subtree trajectories is logarithmically shorter than the length of BBMDP episodes.}
    
    Consider a transition $(s,\, a, \, r,\, s')$. Following Eq.~\eqref{eq:prop2}, the Bellman optimality equation yields a sequence of $\bar{Q}_n$ functions via dynamic programming updates of the form: 
    \begin{align}
        \bar{Q}_{n+1}(o_1, \bar{x}_{o_1}, a) = -2 + \max_{a' \in \mathcal{A}} \bar{Q}_n(o'_1, \bar{x}_{o'_1}, a') + \max_{a'' \in \mathcal{A}} \bar{Q}_n(o'_2, \bar{x}_{o'_2}, a'').
        \label{eq:bell}
    \end{align}
    One can also consider a $k$-step Bellman operator ($k\geq1$), generalizing Eq.~\eqref{eq:bell}. 
    Let $\pi$ be a policy, and $s^{(k)}$ the state reached after a first application of $a$ from $s$, and $k-1$ subsequent applications of $\pi$. 
    Provided the subtree rooted in $o_1$ has not been fathomed within these $k$ time steps, $s^{(k)}$ has $k+1$ new open nodes which we label $o_i^{(k)}$. 
    Then, the Bellman update becomes:
    \begin{equation}
        \bar{Q}_{n+1}(o_1, \bar{x}_{o_1}, a) = -2k + \sum_{i=1}^{k+1} \max_{a' \in \mathcal{A}} \bar{Q}_n(o^{(k)}_i, \bar{x}_{o^{(k)}_i}, a').
    \end{equation}
    This Bellman update can be approximated using a neural network $\mathbf{q}_\theta$, whose weights are trained to minimize the $k$-step temporal difference loss $\mathcal{L}(\theta) = \mathbb{E}[(\mathbf{q}_\theta(o_1, \bar{x}_{o_1}, a) - \bar{Q}_{n+1}(o_1, \bar{x}_{o_1}, a) )^2]$, where $\bar{Q}_{n+1}$ is bootstrapped by the previous $\mathbf{q}_\theta$. 
    Following the work of \citet{farebrother_stop_2024} on training value functions via classification, we also introduce a HL-Gauss cross-entropy loss adapted to the B\&B setting $\mathcal{L}(\theta) = \mathbb{E} [\mathbf{q}_{\theta}(o_1, \bar{x}_{o_1}, a) \cdot \log p_{hist}( \bar{Q}_{n+1}(o_1, \bar{x}_{o_1}, a) ) ]$ where $p_{hist}$ is the function that encodes values in categorical histogram distributions, see Appendix \ref{app:loss} for complete description and theoretical motivation.

    \subsection{BBMDP vs TreeMDP}
    \label{sec:treevsbb}
       As it evacuates the core MDP notions of temporality and sequentiality, TreeMDP fails to describe variable selection in B\&B accurately in the general case.
       This is illustrated in Figure \ref{fig:tree vs clean} and further detailed in Appendix \ref{app:BBvsTree}: although the TreeMDP model is a valid approximation of BBMDP when training an RL agent to minimize temporal difference on one-step trajectories, it produces inconsistent learning schemes when evaluating temporal difference across multi-steps trajectories. 
       In fact, repeatedly applying the tree Bellman operator from \citet{etheve_solving_2021} yields B\&B trees that cannot, in general, be produced by DFS, as shown in Figure \ref{fig:tree vs clean} (d).
       
       Hence, BBMDP leverages the results first established by \citet{etheve_reinforcement_2020} and \citet{scavuzzo_learning_2022} --- in DFS, minimizing the whole B\&B tree size is achieved when any subtree is of minimal size (Eq.~\eqref{eq:pi_star}) --- all while preserving MDP properties.
       Crucially, BBMDP allows to harness RL algorithms that are not compatible with the TreeMDP framework, such as \greent{$k$-step return temporal difference}, TD($\lambda$), or any RL algorithms using MCTS as policy improvement operator \citep{grill_monte-carlo_2020}, as illustrated in Appendix \ref{app:BBvsTree}.
       In the same fashion, BBMDP can be applied to augment the pool of RL algorithms available for learning improved cut selection and primal search heuristics, simply by adapting the action set and the reward model to the task at hand.
       Complete exploration and comparison of such algorithms is beyond the scope of this paper and left for future work, as our primary objective here is to study the conditions under which B\&B search can be cast as an MDP, and provide the community with a solid basis for principled algorithms.

\section{Experiments}
\label{sec:exp}
    We now compare our branching agent against prior IL and RL approaches. 
    For our experiments, we use the open-source solver SCIP 8.0.3 \citep{bestuzheva_scip_2021} as backend \greent{MILP} solver, along with the Ecole library \citep{prouvost_ecole_2020} \greent{both} for instance generation and \greent{environment simulation}. 
    To foster reproducibility, our implementation and pretrained models are made publicly available at \href{https://github.com/abfariah/bbmdp}{https://github.com/abfariah/bbmdp}.
    \begin{table*}[t]
        \centering
        \small
        \caption{Performance comparison of branching agents on four standard MILP benchmarks. For each method, we report total number of B\&B nodes, presolve time and total solving time outside of presolve. Lower is better, \textcolor{purple}{\textbf{red}} indicates best agent overall, \textcolor{blue}{\textbf{blue}} indicates best among RL agents. Presolve is common to all methods. Following prior works, we report geometrical mean over 100 test instances unseen during training and over \greent{100} higher-dimensional transfer instances. \greent{Norm. Score denotes the aggregate average performance obtained by each agent across the four MILP benchmarks, normalized by the score of DQN-BBMDP.}}
        \label{tab:results}
        \begin{adjustbox}{width=\columnwidth,center}
        \begin{tabular}{ccccccccccc} 
            \multicolumn{11}{c}{Test instances} \\
            \toprule
             & \multicolumn{2}{c}{\textbf{Set Covering}} & \multicolumn{2}{c}{\textbf{Comb. Auction}} & \multicolumn{2}{c}{\textbf{Max. Ind. Set}} & \multicolumn{2}{c}{\textbf{Mult. Knapsack}} & \multicolumn{2}{c}{\textbf{Norm. Score}} \\ 
            Method & Node & Time & Node & Time & Node & Time & Node & Time & Node & Time\\
            \toprule
            Presolve & $-$ & $4.74$ & $-$  & $0.90 $ & $-$  & $1.78 $ & $-$  & $0.20$ & $-$ & $-$ \\
            \midrule
            Random & $3289$ & $5.94$ & $1111$ & $2.16$ & $386.8$ & $2.01$ & $733.5$ & $0.55$ & $\greent{995}$ & $\greent{374}$ \\
            SB & $35.8 $ & $12.93$ & $28.2$ & $6.21 $ & $24.9$ & $45.87$ & $161.7$ & $0.69$& $\greent{36}$ & $\greent{2358}$ \\
            SCIP & $62.0 $ & $2.27 $ & $20.2 $ & $1.77$ & $19.5$ & $2.44 $ & $289.5 $ & \textcolor{purple}{$\mathbf{0.53}$} & $\greent{51}$ & $\greent{253}$\\
            \midrule
             IL & \textcolor{purple}{$\mathbf{133.8}$}  & $0.90 $ & \textcolor{purple}{$\mathbf{83.6}$} & $0.65$ & \textcolor{purple}{$\mathbf{40.1}$} & \textcolor{purple}{$\mathbf{0.36}$} & {$272.0$} &{$0.69$} & \textcolor{purple}{$\mathbf{82}$} & \textcolor{purple}{$\mathbf{95}$}\\
             IL-DFS & $136.4$ & \textcolor{purple}{$\mathbf{0.74}$} & $95.5$ & \textcolor{purple}{$\mathbf{0.56}$} & $69.4$ & $0.44$ & $472.8$ & $1.07$ & $\greent{114}$ & \greent{$129$} \\
            \midrule
            PG-tMDP & $649.4 $ & $2.32 $ & $168.0 $ & $0.94 $ & $153.6 $ & $0.92$ & $436.9$ & $1.57$ & \greent{$233$} & \greent{$206$}\\
            DQN-tMDP & $175.8 $ & $0.83 $ & $203.3$ & $1.11$ & $168.0$ & $1.00$ & $266.4$ & $0.73$ & \greent{$151$} & \greent{$136$}\\
            DQN-Retro & $183.0 $ & $1.14 $ & $103.2$ & $0.78$ & $223.0$ & $1.81$ & $250.3$ & $0.67$ & \greent{$137$} & \greent{$160$}\\
            DQN-BBMDP & \textcolor{blue}{$\mathbf{152.3}$} & \textcolor{blue}{$\mathbf{0.77}$} & 
            \textcolor{blue}{$\mathbf{97.9}$} & \textcolor{blue}{$\mathbf{0.62}$} & \textcolor{blue}{$\mathbf{103.2}$} & \textcolor{blue}{$\mathbf{0.69}$} & \textcolor{purple}{$\mathbf{236.6}$} & \textcolor{blue}{$\mathbf{0.66}$} & \textcolor{blue}{$\mathbf{100}$} & \textcolor{blue}{$\mathbf{100}$}\\
            \bottomrule[1.5pt]
        \hspace{5mm}
        \end{tabular} 
        \end{adjustbox}
        \begin{adjustbox}{width=\columnwidth,center}
        \begin{tabular}{ccccccccccc} 
            \multicolumn{11}{c}{Transfer instances} \\
            \toprule
             & \multicolumn{2}{c}{\textbf{Set Covering}} & \multicolumn{2}{c}{\textbf{Comb. Auction}} & \multicolumn{2}{c}{\textbf{Max. Ind. Set}} & \multicolumn{2}{c}{\textbf{Mult. Knapsack}} & \multicolumn{2}{c}{\textbf{Norm. Score}} \\ 
            Method & Node & Time & Node & Time & Node & Time & Node & Time & Node & Time\\
            \toprule
            Presolve & - & $\greent{12.3}$ & - & $\greent{2.67}$ & - & $\greent{5.16}$ & - & $\greent{0.46}$& $-$ & $-$\\
            \midrule
            Random & $271632$  & $842$ &  $317235$  & $749$ & $215879$ & $2102$ & $93452$ & $70.6$& \greent{$5555$} & \greent{$2737$}\\
            SB & $672.1$ & $398$ & $389.6$  & $255$ & $169.9$ & $2172$ & \textcolor{purple}{$\mathbf{1709}$} & \textcolor{purple}{$\mathbf{12.5}$} & \greent{$9$} & \greent{$1425$}\\
            SCIP & $3309$  & $48.4$ &  $1376$  & $14.77$ & $3368$ & $90.0$ & $30620$ & $22.1$ & \greent{$62$} & \greent{$90$}\\
            \midrule
             IL & \textcolor{purple}{$\mathbf{2610}$}  & {$23.1$} &  \textcolor{purple}{$\mathbf{1309}$}  & \textcolor{purple}{$\mathbf{9.4}$} & \textcolor{purple}{$\mathbf{1882.0}$} & \textcolor{purple}{$\mathbf{38.6}$} & $9747$ & $43.5$ & \textcolor{purple}{$\mathbf{39}$} & \textcolor{purple}{$\mathbf{54}$}\\
             IL-DFS & $3103 $ & \textcolor{purple}{$\mathbf{22.5}$} & $1802 $ & $10.2 $ & $3501$ & $51.9$& $43224$ & $131$ & \greent{$75$} & \greent{$80$}\\
            \midrule
            PG-tMDP & $44649 $ & $221 $ & $6001 $ & $30.7 $ & \textcolor{blue}{$\mathbf{3133}$} & \textcolor{blue}{$\mathbf{39.5}$} & $35614$ & $123$ & \greent{$298$} & \greent{$223$}\\
            DQN-tMDP & $8632$  & $71.3$ &  $20553$  & $116$ & $45634$ & $477$ & \textcolor{blue}{$\mathbf{22631}$} & \textcolor{blue}{$\mathbf{65.1}$} & \greent{$439$} & \greent{$445$}\\
            DQN-Retro & $6100$  & $59.4$ &  $2908$  & $18.4$ & $119478$ & $1863$ & $27077$ & {$79.5$} & \greent{$494$} & \greent{$662$}\\
            DQN-BBMDP & \textcolor{blue}{$\mathbf{5651}$}  & \textcolor{blue}{$\mathbf{46.4}$} & \textcolor{blue}{$\mathbf{2273}$} & \textcolor{blue}{$\mathbf{11.8}$} & $7168$ & $81.3$ & $37098$ & $109$ & \textcolor{blue}{$\mathbf{100}$} & \textcolor{blue}{$\mathbf{100}$}\\
            \bottomrule[1.5pt]
        \end{tabular}
        \end{adjustbox}
    \end{table*}

    \subsection{Experimental setup}
    \label{sec:exp_setup}
    \textbf{Benchmarks.}   We consider the usual standard MILP benchmarks for learning branching strategies: set covering, combinatorial auctions, maximum independent set and multiple knapsack problems. 
    We train and test on instances of same dimensions as \citet{scavuzzo_learning_2022} and \citet{parsonson_reinforcement_2022}, see Appendix \ref{app:milps}. 
    As to SCIP configuration, as in previous work, we set the time limit to one hour, disable restart, and deactivate cut generation beyond root node. 
    All the other parameters are left at their default value.
    
    \textbf{Baselines.}   We compare our DQN-BBMDP agent against DQN-TreeMDP (DQN-tMDP) \citep{etheve_reinforcement_2020}, REINFORCE-TreeMDP (PG-tMDP) \citep{scavuzzo_learning_2022} and the current state-of-the-art DQN-Retro \citep{parsonson_reinforcement_2022} agents. 
    We also compare against the reference IL expert from \citet{gasse_exact_2019} as well as IL-DFS, which is the same expert, only evaluated following a depth-first search node selection policy.
    More details on these baselines can be found in Appendix \ref{app:baseline}.
    Finally, we report the performance of reliability pseudo cost branching (SCIP), the default branching heuristic used in SCIP, strong branching (SB) \citep{applegate_finding_1995}, the greedy expert from which the IL agent learns from, and random branching (Random), which randomly selects a fractional variable. 

    \textbf{Network architecture.} Following prior work, we use the bipartite graph representation introduced by \citet{gasse_exact_2019} augmented by the features proposed in \citet{parsonson_reinforcement_2022} to represent B\&B nodes.
    Additionally, we use the graph convolutional architecture of \citet{gasse_exact_2019} to parameterize our $Q$-value network; see Appendix \ref{app:features} for a more detailed description.
    
    
    \textbf{Training \& evaluation.} Models are trained on instances of each benchmark separately, and evaluated on test instances and transfer instances. Validation curves can be found in Appendix \ref{app:training_curves}.
    For evaluation, we report the node and time performance over 100 test instances unseen during training, as well as on \greent{100} transfer instances of higher dimensions (see Table \ref{tab:milps} in Appendix \ref{app:milps}). 
    \greent{At evaluation, performance scores are averaged over 5 seeds.}
    Importantly, when comparing a machine learning (IL or RL) branching strategy with a standard SCIP heuristic, time performance is the only relevant criterion.
    In fact, when implementing one of its own branching rules, SCIP triggers a series of techniques strengthening the current MILP formulation. If these techniques effectively reduce the number of nodes to visit, they incur computational overhead which ultimately increases SCIP overall solving time. This renders node comparisons between ML and non-ML branching strategies negligible relative to solving time evaluations, as observed by \citet{gamrath2018measuring, scavuzzo_learning_2022}.

    \begin{table*}[t]
        \centering
        \small
        \caption{Ablation impact of BBMDP, HL-Gauss loss and DFS. We remove one component one at the time, and evaluate corresponding versions on 100 set covering test instances after training for $200$k gradient steps as described in section \ref{sec:exp_setup}.}
        \label{tab:ablation}
        \begin{tabular}{c|ccccc} 
            $k$-step return & DQN-BBMDP &  w.o. DFS & w.o. HL-Gauss & w.o. BBMDP & DQN-TreeMDP \\
            \toprule
            $k = 1$ & $158.9 $ & $156.2\, (-2\%)$ & $175.8 (+10\%)$ & $158.9 (+0\%)$ & $175.8 (+10\%)$ \\
            $k = 3$ & $\mathbf{152.3} (-4\%) $ & $150.1 \, (-5\%)$ & $172.3 \, (+8\%)$ & $162.1 \, (+2\%)$  & $178.9 \,(+13\%)$    \\
            \bottomrule
        \end{tabular}
    \end{table*}

    \subsection{Results}
     Computational results obtained on the four benchmarks are presented in Table \ref{tab:results}. 
     On test instances, \textbf{DQN-BBMDP consistently obtains best performance} among RL agents. When compared against prior state-of-the-art DQN-Retro, DQN-BBMDP achieves an \greent{aggregate} average 27\% reduction of total number of node and 38\% reduction of solving time outside presolve \greent{across the four Ecole benchmarks}, as reported in Figure \ref{fig:teasing}.
     Contrary to \citet{parsonson_reinforcement_2022}, we find DQN-Retro to yield performance comparable to DQN-tMDP. Remarkably, all RL agents outperform the SCIP solver on 3 out of 4 benchmarks in terms of solving time. 
     
     On transfer instances, DQN-BBMDP also dominates among RL agents, although it is outperformed by PG-tMDP on maximum independent set instances and by DQN-Retro on multiple knapsack instances. \greent{The aggregate performance gap between DQN-BBMDP and other RL baselines is notably wider on transfer instances, which aligns with the advantages of using a principled MDP formulation over TreeMDP.}
     \greent{In fact, DQN-BBMDP is the first RL agent to demonstrate robust generalization capabilities on higher dimensional instances, outperforming SCIP on 3 out of 4 benchmarks.}

     Additional performance metrics are provided in Appendix \ref{app:results}. 
     A notable metric is the number of wins, \emph{i.e.} the number of test instances that one algorithm solves faster than any other baseline. 
     DQN-BBMDP outperforms even the IL approaches both in average solving time and total number of wins, on the set covering, combinatorial auction and multiple knapsack test instances.


    
    \subsection{Ablation study}
    \label{sec:ablation}
    We perform an ablation study on set covering test instances to separate the performance gain associated with BBMDP and the HL-Gauss classification loss. 
    Since BBMDP and TreeMDP are strictly equivalent when minimizing one-step temporal difference, 
    we evaluate the performance gap between one-step and $k$-step TD learning for both DQN-BBMDP and DQN-TreeMDP.
    
    As shown in Table \ref{tab:ablation}, we find that the bulk of the performance gain is brought by the use of a cross-entropy loss. 
    Nonetheless, we observe that the use of a multi-step TD loss improves the performance of DQN-BBMDP, but degrades the performance of DQN-TreeMDP. This supports our initial claim that, in the general case, despite improving the empirical properties of RL algorithms, TreeMDP introduces approximations which hinder the asymptotic performance achievable by RL agents.
    Following \citet{parsonson_reinforcement_2022}, we also evaluate the cost of opting for depth-first search instead of best estimate search, SCIP's default node selection policy, when learning branching strategies.
    Contrary to their work, we find DFS not to be restrictive in practice in terms of performance. We further investigate theses discrepancies in Appendix \ref{app:parsonson}.

\section{Conclusion and perspectives}
\label{sec:conclusion}

Guiding combinatorial optimization (CO) algorithms with reinforcement learning (RL) has proven challenging, including beyond mixed-integer linear programming (MILP) \citep{berto_rl4co_2023}.
Not only are RL agents consistently outperformed by human-expert CO heuristics or IL agents trained to mimic these experts, but their application has also been limited so far to fairly easy problem instances. 
This begs for a thorough study of these problems' properties, for well-posed formulations and principled methods. 
In this work, we showed the theoretical and practical limits of the concept of TreeMDP for learning optimal branching strategies in MILP, highlighting in which context TreeMDPs were a valid formulation.
Introducing BBMDP, we proposed a rigorous description of variable selection in B\&B, unlocking the use of vanilla dynamic programming methods.
In turn, the resulting DQN-BBMDP method outperformed prior RL-inspired agents on the Ecole benchmark. 

Nonetheless, there remains significant room for improvement for RL approaches on both test and transfer instances, suggesting that the generalization capacity of RL agents still lags behind that of IL. 
We believe this to be mainly due to out of distribution effects, since transfer instances are intrinsically more complex than training instances, and $Q$-networks are lead to evaluate B\&B nodes with subtree sizes far exceeding those encountered during training. In contrast, IL networks learn to predict the strong branching action via behavioral cloning, making them robust to such out-of-distribution scaling effects.
Yet, we believe that building on a principled MDP formulation of variable selection in B\&B is a stepping stone to achieve substantial acceleration of solving time for higher-dimensional MILPs in the future, also opening avenues to tackle distribution shift and domain adaptation. 



\bibliography{references}
\bibliographystyle{plainnat}


\newpage
\appendix

\section{BBMDP vs TreeMDP}
\subsection*{Side-by-side comparison}
\label{app:BBvsTree}
TreeMDP offers an intuitive and computationaly efficient framework for training RL agents to learn enhanced variable selection strategies in branch and bound. However, by discarding core MDP concepts such as temporality and sequentiality, it fails to accommodate the diversity of RL algorithms.
For example, defining the $k$-step return temporal difference target in TreeMDP is challenging. Since branching on action $a_i$ at state $s_i$ produces two child next states $s_i^-$ and $s_i^+$, applying recursively the tree Bellman operator from \citet{etheve_solving_2021} to $V^\pi$ yields $V^\pi(s_i) = -\sum_{j=1}^{k} 2^j + \sum_{j=1}^{2^k} V^\pi(s^j_i)$.
The corresponding TreeMDP trajectory is depicted in Figure \ref{fig:tree vs clean} (d) for $k=3$.
More generally, it is important to note that such trajectories cannot be obtained when solving a MILP using B\&B with a depth-first search node selection policy, since the subtree rooted in $s_i^-$ must be fully fathomed before the node $s_i^+$ can be considered for expansion. Therefore, the TreeMDP $k$-step return TD target defined above stems from approximations in the B\&B dynamics, which undermines the best performance achievable by $k$-step tMDP-DQN, as demonstrated in the ablation study in Section \ref{sec:ablation}.
Table \ref{tab:BBvsTree} further highlights the key differences between BBMDP and TreeMDP.

\begin{table}[h]
    \caption{Side by side comparison of BBMDP and TreeMDP frameworks, complementing Figure \ref{fig:tree vs clean}.}
    \centering
    \small
    \renewcommand{\arraystretch}{2} 
    \begin{tabular}{ccc} 
        & BBMDP & TreeMDP \\
        \toprule
        MDP & Yes & No \\
        
        State $s$ & $s_t = (\mathcal{V}_t, \mathcal{E}_t, \bar{x}_t)$ & $s_i = (P_{i}, x^{*}_{LP, i}, \bar{x}_{i})$ \\
        
        Action  $a$ & $a_t \in \mathcal{I}$ & $a_i \in \mathcal{I}$ \\
        
        Reward $r$ & -2 & -1 \\
        
        Next state $s'$ & $s_{t+1} = (\mathcal{V}_{t+1}, \mathcal{E}_{t+1}, \bar{x}_{t+1})$ & $s_i^{-}, s_i^+$ \\
        
        $k$-step next state $s^{(k)}$ & $s_{t+k} = (\mathcal{V}_{t+k}, \mathcal{E}_{t+k}, \bar{x}_{t+k})$ & $s_i^1$, $s_i^2$, ..., $s_i^{2^k}$\\
        
        $\mathcal{B}(V^\pi) = V^\pi$ & \scriptsize{$\bar{V}^{\pi}(o_1, \bar{x}_{o_1}) = -2 + \bar{V}^{\pi}(o'_1, \bar{x}_{o'_1}) + \bar{V}^{\pi}(o'_2, \bar{x}_{o'_2})$} & \scriptsize{$V^\pi(s_i) = -2 + V^\pi(s^-_i) + V^\pi(s^+_i)$} \\
        
        $\mathcal{B}^k(V^\pi)= V^\pi$ & \scriptsize{$\bar{V}^\pi(o_1, \bar{x}_{o_1}) = -2k + \sum_{i=1}^{k+1} \bar{V}^\pi(o^{(k)}_i, \bar{x}_{o^{(k)}_i})$} & \scriptsize{$ V^\pi(s_i) = -\sum_{j=1}^{k} 2^j + \sum_{j=1}^{2^k} V^\pi(s^j_i)$}\\
        \bottomrule
    \end{tabular}

    \label{tab:BBvsTree}
\end{table}

\subsection*{Learning to branch with MCTS}

MCTS-based RL algorithms have achieved remarkable performance in combinatorial settings, particularly in board games \cite{silver_general_2018}. Interstingly, by providing a rigorous MDP formulation for variable selection in B\&B, BBMDP enables the adapation of MCTS-based learning algorithms to the B\&B framework, overcoming the incompatibilities that persisted within TreeMDP. In fact, when searching for the next node to expand, MCTS traverses the tree following a UCT exploration criterion: 

\begin{equation}
    a^k = \arg \max_{a \in \mathcal{A}} \left[Q(s,a) + \pi(s, a) \cdot\frac{\sqrt{\sum_{b\in \mathcal{A}}{N(s,b)}}}{1+N(s,a)}\left(c_1+\log(\frac{\sum_{b\in \mathcal{A}} N(s,b)}{c_2}) \right)\right]
\end{equation}

where $N(s,a)$ corresponds to the number of visits of the node $(s,a)$, and $c_1, c_2$ are search hyperparameters. This formulation balances exploration and exploitation, guiding the search toward promising states while ensuring sufficient exploration of the tree. However in TreeMDP, such an exploration criteria leads to inconsistencies, as the exploration between node $s_i^{-}$ and $s_i^{+}$ is balanced based on their estimated value. Since both $s_i^{-}$ and $s_i^{+}$ originate from the same branching decision $a_i$, differentiating between them for expansion is irrelevant. In B\&B, branching on $a_i$ implies that $s_i^-$ and $s_i^+$ are both necessarily added to the B\&B tree. Consequently, in TreeMDP, MCTS could guide the search towards a suboptimal action $a_i$ based on the assumption that it produces a state $s_i^{-}$ of small subtree size, while failing to account for the potentially enormous subtree size of $s_i^{+}$. This fundamental inconsistency hinders the effective application of MCTS within the TreeMDP framework.

\section{HL-Gauss Loss}
\label{app:loss}
As they investigated the uneven success met by complex neural network architectures such as Transformers in supervised versus reinforcement learning, \citet{farebrother_stop_2024} found that training agents using a cross-entropy classification objective significantly improved the performance and scalability of value-based RL methods. However, replacing mean squared error regression with cross-entropy classification requires methods to transform scalars into distributions and distributions into scalars. \citet{farebrother_stop_2024} found the Histogram Gaussian loss (HL-Gauss) \citep{imani_improving_2018}, which exploits the ordinal structure of the regression task by distributing probability mass on multiple neighboring histogram bins, to be a reliable solution across multiple RL benchmarks. Concretely, in HL-Gauss, the support of the value function $\mathcal{Z}_v \subset \mathbb{R}$ is divided in $m_b$ bins of equal width forming a partition of $\mathcal{Z}_v$. Bins are centered at $\zeta_i \in \mathcal{Z}_v$ for ${1 \leq i \leq m_b}$, we use $\eta = (\zeta_{max}-\zeta_{min})/m_b$ to denote their width. Given a scalar $z \in \mathcal{Z}_v$, we define the random variable $Y_z \sim \mathcal{N}(\mu=z, \sigma^2)$ and note respectively $\phi_{Y_z}$ and $\Phi_{Y_z}$ its associate probability density and cumulative distribution function. $z$ can then be encoded into a histogram distribution on $\mathcal{Z}_v$ using the function $p_{hist}: \mathbb{R} \rightarrow [0, 1]^{m_b}$. Explicitly, $p_{hist}$ computes the aggregated mass of $\phi_{Y_z}$ on each bin:

\[p_{hist}(z) = (p_i(z))_{1\leq i \leq m_b} \text{ with } p_i(z) = \int_{\zeta_i-\frac{\eta}{2}}^{\zeta_i+\frac{\eta}{2}} \phi_{Y_z}(y) dy = \Phi_{Y_z}(\zeta_i+\frac{\eta}{2}) - \Phi_{Y_z}(\zeta_i-\frac{\eta}{2}) \]

Conversely, histogram distributions $(p_i)_{1 \leq i \leq m_b}$ such as the ones outputted by agents' value networks can be converted to scalar simply by computing the expectation: $z = \sum_{i=1}^{m_b} p_i \cdot \zeta_i$. \\
\newline BBMDP is a challenging setting to adapt HL-Gauss, as the support for value functions spans over several order of magnitude. In practice, we observe that for train instances of the Ecole benchmark, $\mathcal{Z}_v = [-10^6, -2]$. Since value functions predict the number of node of binary trees built with B\&B, it seems natural to choose bins centered at $\zeta_i = -2^{i}$ to partition $\mathcal{Z}_v$. In order to preserve bins of equal size, we consider distributions on the support $\psi(\mathcal{Z}_v)$ with $\psi(z) = \log_2(-z)$ for $z \in \mathcal{Z}_v$, such that $\psi(\mathcal{Z}_v)$ is efficiently partitioned by bins centered at $\zeta_i=i$ for ${1 \leq i \leq m_b}$.  Thus, in BBMDP histograms distributions are given by $p_{hist}(z) = (p_i \circ \psi(z))_{1 \leq i \leq m_b}$ for $z \in \mathcal{Z}_v$, and can be converted back to $\mathcal{Z}_v$ through $z = \sum_{i=1}^{m_b} p_i \cdot \psi^{-1}(\zeta_i)$ with $\psi^{-1}(z) = -2^{z}$.

\section{Baselines}
\label{app:baseline}

\paragraph{Imitation learning} We trained and tested IL agents using the official Ecole re-implementation of \citet{gasse_exact_2019} shared at \url{https://github.com/ds4dm/learn2branch-ecole/tree/main}.

\paragraph{DQN-TreeMDP} Since there is no publicly available implementation of \citet{etheve_reinforcement_2020}, we re-implemented DQN-TreeMDP and trained it on the four Ecole benchmarks, using when applicable the same network architectures and training parameters as in DQN-BBMDP and DQN-Retro. We share implementation and trained network weights to the community.

\paragraph{PG-tMDP} We used the official implementation of \citet{scavuzzo_learning_2022} to evaluate PG-TreeMDP. For each benchmark, we used the tMDP+DFS network weights shared at \url{https://github.com/lascavana/rl2branch}.

\paragraph{DQN-Retro} As \citet{parsonson_reinforcement_2022} only trained on set covering instances, we took inspiration from the official implementation shared at \url{https://github.com/cwfparsonson/retro_branching} to train and evaluate DQN-Retro agents on the four Ecole benchmarks. Importantly, we trained and tested DQN-Retro following SCIP's default node selection strategy, see Appendix \ref{app:parsonson} for more details. We share our re-implementation and trained network weights with the community.

Importantly, on the multiple knapsack benchmark, given the high computational cost associated with opting for a depth-first search node selection policy, all baselines excepted IL-DFS are evaluated following SCIP default node selection policy.

\section{\greent{BBMDP vs Retro branching}}
\label{app:parsonson}
\greent{In their work, \citet{parsonson_reinforcement_2022} proposed to train RL agents on retrospective trajectories built from TreeMDP episodes, in order to leverage the state-of-the art node selection policies implemented in MILP solvers}. When reproducing their work, we found several discrepancies with the results they stated. First, the  performance gap between DQN-Retro and DQN-TreeMDP \citep{etheve_reinforcement_2020} turned out to be much narrower than expected. On test set covering instances, the only benchmark on which the two agents are compared in \citet{parsonson_reinforcement_2022}, we even found DQN-TreeMDP to perform better. Second, \citet{parsonson_reinforcement_2022} found that adopting a best-first-search (BeFS) node selection strategy at evaluation time greatly improved the performance of DQN-TreeMDP on test set covering instances, \greent{indicating that abandoning DFS during training could be beneficial. However,} in our experiments, we observed a 20\% performance drop when replacing DFS by BeFS at evaluation time. After thorough examination of both \citet{parsonson_reinforcement_2022}'s article and implementation, we found that the baseline labeled as DQN-TreeMDP (FMSTS-DFS in their article) was quite distant from the branching agent originally described in \citep{etheve_reinforcement_2020}. In fact, in \citet{parsonson_reinforcement_2022}, the \citet{etheve_reinforcement_2020} branching agent is not trained on TreeMDP trajectories, but on retrospective trajectories built from TreeMDP episodes, using a DFS construction heuristic. Therefore, \citet{parsonson_reinforcement_2022} could not conclude on the superiority of retro branching over TreeMDP, \greent{nor could they assess the limitations of DFS-based RL agents. In contrast, our contribution provides compelling evidence that, while DFS is generally expected to hinder the training performance of RL agents due to its reputation as a suboptimal node selection policy, the theoretical guarantees brought by DFS in BBMDP enable to surpass prior state-of-the-art non-DFS agents. We believe this is because optimizing the node selection policy has less influence on tree size performance compared to optimizing the variable selection policy, as evidenced by \citet{etheve_solving_2021}}.

\section{Neural network}
\label{app:features}

\paragraph{State representation}
Following the works of \citet{gasse_exact_2019}, MILPs are best represented by bipartite graphs $\mathcal{G} = (\mathcal{V}_{\mathcal{G}}, 
\mathcal{C}_{\mathcal{G}}, \mathcal{E}_{\mathcal{G}})$ where 
$\mathcal{V}_{\mathcal{G}}$ denotes the set of variable nodes, 
$\mathcal{C}_{\mathcal{G}}$ denotes the set of constraint nodes, and 
$\mathcal{E}_{\mathcal{G}}$ denotes the set of edges linking variable and 
constraints nodes.
Nodes $v_{\mathcal{G}} \in \mathcal{V}_{\mathcal{G}}$ and $c_{\mathcal{G}} \in \mathcal{C}_{\mathcal{G}}$ are connected if the variable associated with $v_{\mathcal{G}}$ appears in the constraint associated with $c_{\mathcal{G}}$.
Given a MILP $P$, defined as in Section \ref{sec:b&b}, its associate bipartite representation $\mathcal{G}$ has $|\mathcal{G}| = |\mathcal{V}_{\mathcal{G}}| + |\mathcal{C}_{\mathcal{G}}| = n + m$ nodes. We use bipartite graphs to encode the information contained in $(o_i, \bar{x}_{o_i})$, as defined in section \ref{sec:core}.
In our experiments, IL and PG-tMDP agents use the list of features of \citet{gasse_exact_2019} to represent variable nodes, constraint nodes and edges, while DQN-BBMDP, DQN-TreeMDP and DQN-Retro agents also make use of the additional features introduced by \citet{parsonson_reinforcement_2022}.

\paragraph{Network architecture} All RL agents utilize the graph convolutional network architecture described in \citet{scavuzzo_learning_2022} and \citet{parsonson_reinforcement_2022}. In DQN-BBMDP, the architecture differs slightly, with the final layer outputting distribution vectors in $\mathbb{R}^{m_b}$ instead of scalar values in $\mathbb{R}$.

\section{Instance dataset}
\label{app:milps}
    Instance datasets used for training and evaluation are decribed in Table \ref{tab:milps}. We trained and tested on instances of same dimensions as \citet{scavuzzo_learning_2022} and \citet{parsonson_reinforcement_2022}. As a reminder, the size of action set  $\mathcal{A}$ is equal to the number of integer variables in $P$. Consequently, action set sizes in the Ecole benchmark range from 30 to 480 for train / test instances and from 50 to 980 for transfer instances. \\
    \begin{table}[h]
        \centering
        \caption{Instance size for each benchmark. Performance is evaluated on test instances that match the size of the training instances, as well as on larger instances, to further assess the generalization capacity of our agents. Last two columns indicate the approximate number of integer variables after presolve, both for train / test and transfer instances.}
        \small
        \begin{tabular}{ccccccc} 
            & & & \multicolumn{2}{c}{Parameter value} & \multicolumn{2}{c}{\# Int. variables} \\
            Benchmark & Generation method & Parameters & Train / Test  & Transfer & Train / Test & Transfer \\ 
            \toprule
            \begin{tabular}{@{}c@{}}Combinatorial \\ auction \end{tabular} & \scriptsize{\citet{leyton2000towards}} & \begin{tabular}{@{}c@{}}Items \\ Bids\end{tabular} & \begin{tabular}{@{}c@{}}100 \\ 500\end{tabular} & \begin{tabular}{@{}c@{}}200 \\ 1000\end{tabular} & $ 100 $ & $200$\\
            \midrule
            Set covering & \scriptsize{\citet{balas1980set}} &\begin{tabular}{@{}c@{}}Items \\ Sets\end{tabular}& \begin{tabular}{@{}c@{}}500 \\ 1000\end{tabular} & \begin{tabular}{@{}c@{}}1000 \\ 1000\end{tabular} & $ 100$ & $ 130$\\
            \midrule
            \begin{tabular}{@{}c@{}}Maximum \\ independent set\end{tabular}  & \scriptsize{\citet{bergman2016decision}} & Nodes & 500 & 1000 & $480$ & $ 980$ \\
            \midrule
            \begin{tabular}{@{}c@{}}Multiple \\ knapsack\end{tabular} & \scriptsize{\citet{fukunaga2011branch}} & \begin{tabular}{@{}c@{}}Items\\ Knapsacks\end{tabular} & \begin{tabular}{@{}c@{}}100 \\ 6\end{tabular} & \begin{tabular}{@{}c@{}}100 \\ 12\end{tabular} & $ 30 $ & $ 50$ \\
            \bottomrule
        \hspace{5mm}
        \end{tabular}
        \label{tab:milps}
    \end{table}

\section{\greent{Training pipeline}}
\label{app:rl}
\paragraph{DQN Implementation} In Algorithm \ref{alg:corrected_fmsts}, we provide a \greent{description of DQN-BBMDP training pipeline}. Our DQN implementation includes several Rainbow-DQN features \citep{hessel_rainbow_2017}: double DQN \citep{van2016deep}, $n$-step learning and prioritized experience replay (PER) \citet{schaul2015prioritized}. Moreover, as DQN-BBMDP learns distributions representing $Q$-values, it integrates elements of \citep{bellemare2017distributional}. \\

\begin{algorithm}
\caption{\greent{DQN-BBMDP}}\label{alg:corrected_fmsts}
    \begin{algorithmic}
    \FOR{$t = 0...N-1$}
        \STATE Draw randomly an instance $P \sim p_0$.
        \STATE Solve $P$ by acting following a combined $\epsilon$-greedy and Boltzman exploration according to $\mathbf{q}_{\theta_t}$.
        \STATE Collect transitions along the generated tree $(s_i, a_i, \sum_{j=1}^{k} r_{i+j}, s_{i+k})$ and store them into a replay buffer $\mathcal{B}_{replay}$.
        \STATE Update $\theta_t$ using the loss described in Sec~\ref{sec:bellman} on transition batches drawn from $\mathcal{B}_{replay}$.
    \ENDFOR
    \end{algorithmic}
\end{algorithm}

\paragraph{Exploration} We train our agents following Boltzmann and $\epsilon$-greedy exploration combined. Concretely, agents select actions uniformly from $\mathcal{A}$ with probability $\epsilon$, while following a Boltzmann exploration strategy with temperature $\tau$ for the remaining probability $1-\epsilon$. \greent{The decay rates for $\epsilon$ and $\tau$ are listed in Table \ref{tab:hyperparams}.}

\greent{\paragraph{Reward model} In Section \ref{sec:def}, we defined $\mathcal{R}(s,a) =-2$ for all transition, so that the overall value of a trajectory matched the size of the B\&B tree. In practice, all negative constant reward model yield equivalent optimal policies in BBMDP, therefore, we chose to implement $\mathcal{R}(s,a) = -1$ for all RL baselines in order to allow clearer comparison between BBMDP and TreeMDP agents.}

\paragraph{Training parameters} Table \ref{tab:hyperparams} provides the list of hyperparameters used to train DQN agents on the Ecole benchmarks. To allow fair comparisons, when applicable, we keep SCIP parameters, training parameters and network architectures fixed for all DQN-agents.

\begin{table}[h]
    \small
    \centering
    \caption{Training parameters for all DQN branching agents. For DQN-Retro, we take $\gamma = 0.99$ as in \citet{parsonson_reinforcement_2022}.}
    \hspace{5mm}
    \begin{tabular}{ccc} 
        Module & Training parameter &  Value \\
        \toprule
        & Batch size & $128$ \\
        & Optimizer & Adam \\
        & $k$-step return & $3$ \\
       $Q$-learning & Learning rate $l_r$ & $5 \times 10^{-5}$ \\
        & Discount factor $\gamma$ & $1.0$ \\
        & Agent steps per network update & 10 \\
         & Soft target network update $\tau_{net}$ & $10^{-4}$ \\
        \midrule
        Replay buffer & Buffer minimum size $|\mathcal{B}_{replay}|_{init}$ & $20 \times 10^3$ \\
        & Buffer maximum capacity $|\mathcal{B}_{replay}|_{max}$ & $100 \times 10^3$ \\
        \midrule
        & PER $\alpha$ & $0.6$ \\
        & PER $\beta_{init}$ & $0.4$ \\
        Prioritized experience replay & PER $\beta_{final}$ & $1.0$\\
        & $\beta_{init} \rightarrow \beta_{final}$ learner steps & $100 \times 10^3$\\
        & Minimum experience priority & $10^{-3}$ \\
        \midrule
        & Start exploration probability $\epsilon_{init}$ & $1.0$ \\
        & Minimum exploration probability $\epsilon_{min}$ & $2.5 \times 10^{-2}$ \\
        Exploration & $\epsilon$-decay & $10^{-4}$ \\
        & Start temperature $\tau_{init}$ & $1.0$ \\
        & Minimum temperature $\tau_{min}$ & $10^{-3}$ \\
        & $\tau$-decay & $10^{-5}$ \\
        \midrule
        & $z_{min}$ & -1 \\
        HL-Gauss & $z_{max}$ & $16$ \\
        (only for DQN-BBMDP)& $m_b$ &$18$\\
        & $\sigma$ & $0.75$  \\
        \bottomrule
    \end{tabular}
\label{tab:hyperparams}
\end{table}

\section{Validation curves}
\label{app:training_curves}

All experiments were conducted on an NVIDIA DGX A100 system equipped with 8× A100 40GB GPUs, 2× AMD EPYC 7742 64-core CPUs (128 threads total), and 1 TB of DDR4 RAM. 
We present validation curves for DQN-BBMDP, DQN-TreeMDP and DQN-Retro in Figure \ref{fig:four_subfigures}. For each benchmark we trained for $200$k gradient steps, which took approximately $2$ days for combinatorial auction instances, $3$ days for set covering instances, $5$ days for multiple knapsack instances and $7$ days for maximum independent set instances. As shown in Figure \ref{fig:four_subfigures}, DQN-BBMDP training was interrupted before final convergence on 3 out of 4 benchmarks, hinting that performance could likely be improved by training for more steps. \\
\newline

\begin{figure}[htbp]
    \centering
    \includegraphics[width=\textwidth]{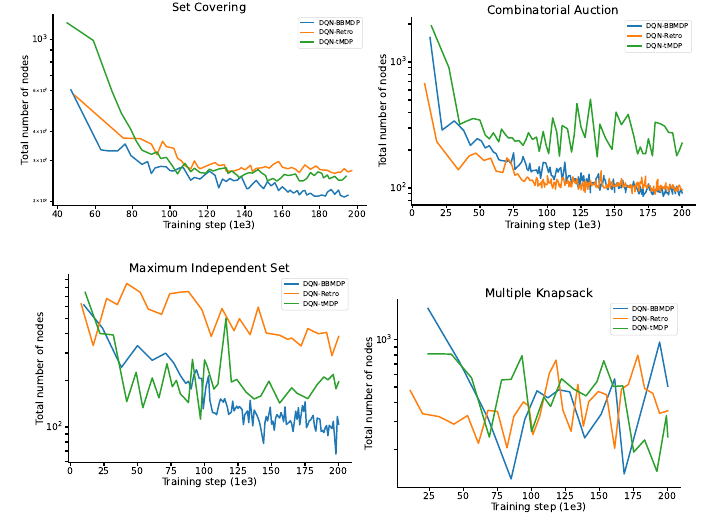}
    \caption{Validation curves for DQN-BBMDP, DQN-Retro and DQN-tMDP agents, in log scale. Throughout training, agents are evaluated on 20 validation instances after each batch of 100 training instances solved. Note that on the multiple knapsack benchmark, none of the agents reach convergence.}
    \label{fig:four_subfigures}
\end{figure}

\section{\greent{Additional computational results }}
In this section, we include further computational results on instances of the Ecole benchmark.

Table \ref{tab:win_results} provides additional performance metrics to compare the different baselines across test and transfer benchmarks. For each benchmark, we report the number of wins and the average rank of each baseline across 100 evaluation instances. The number of wins is defined as the number of instances where a baseline solves a MILP problem faster than any other baseline. When multiple baselines fail to solve an instance to optimality within the time limit, their performance is ranked based on final dual gap. 

Finally, Table \ref{tab:complete_results} recapitulates the computational results presented in Table \ref{tab:results}, and provides for each baseline the per-benchmark standard deviation over five seeds, as well as the fraction of test instances solved to optimality within the time limit. \\

    \begin{table}[]
        \centering
        \caption{Additional performance metrics for each baseline on train / test and transfer instance benchmarks, see Appendix \ref{app:milps} for instance details. For each benchmark, we report the number of wins, and the average rank of each baseline across the 100 evaluation instances. We also report for each baseline the fraction of test instances solved to optimality within time limit. The number of wins is defined as the number of instances where a baseline solves a MILP problem faster than all other baselines. When multiple baselines fail to solve an instance to optimality within time limit, their performance is ranked based on final dual gap. }
        \hspace{5mm}
        \small
        \begin{tabular}{ccccccc} 
            & \multicolumn{3}{c}{Train / Test} & \multicolumn{3}{c}{Transfer}\\
            Method & Solved & Wins & Rank & Solved &  Wins & Rank   \\  
            \toprule
            SCIP  & $100/100$ & $8/100$ & $5.7$ & $100/100$ & $10/100$ & $3.3$ \\
            \midrule
             IL & $100/100$ & $1/00$ & $3.8$ & $100/100$ & $29/100$ & $1.8$ \\
             IL-DFS & $100/100$ & $35/100$ & $2.1$ & $\textcolor{purple}{\mathbf{100/100}}$ & $\textcolor{purple}{\mathbf{58/100}}$ & $\textcolor{purple}{\mathbf{1.7}}$ \\
            \midrule
            PG-tMDP  & $100/100$ & $0/100$ & $6.6$ & $78/100$ & $0/100$ & $6.8$ \\
            DQN-tMDP  & $100/100$ & $11/100$ & $2.9$ & $96/100$ & $0/100$ & $5.0$ \\ 
            DQN-Retro & $100/100$ & $1/100$ & $4.9$ & $98/100$ & $0/100$ & $5.1$ \\
            DQN-BBMDP & $100/100$ & $\textcolor{purple}{\mathbf{44/100}}$ & $\textcolor{purple}{\mathbf{1.9}}$ & $\textcolor{blue}{\mathbf{100/100}}$ & $\textcolor{blue}{\mathbf{3/100}}$ & $\textcolor{blue}{\mathbf{4.2}}$ \\
            \bottomrule
            & \multicolumn{6}{c}{Set covering}
        \hspace{1.5cm}
        \end{tabular}
            \begin{tabular}{ccccccc} 
            & \multicolumn{3}{c}{Train / Test} & \multicolumn{3}{c}{Transfer}\\
            Method & Solved & Wins & Rank & Solved &  Wins & Rank  \\ 
            \toprule
            SCIP  & $100/100$ & $14/100$ & $4.1$ & $100/100$ & $14/100$ & $3.51$ \\
            \midrule
             IL & $100/100$ & $16/100$ & $3.5$ & $\textcolor{purple}{\mathbf{100/100}}$ & $\textcolor{purple}{\mathbf{47/100}}$ & $\textcolor{purple}{\mathbf{1.7}}$\\
             IL-DFS & $100/100$ & $31/100$ & $2.6$ & $100/100$ & $20/100$ & $2.5$  \\
            \midrule
            PG-tMDP  & $100/100$ & $0/100$ & $5.7$ & $100/100$ & $0/100$ & $5.8$\\
            DQN-tMDP  & $100/100$ & $0/100$ & $6.2$ & $100/100$ & $0/100$ & $6.5$ \\ 
            DQN-Retro & $100/100$ & $10/100$ & $3.6$ & $100/100$ & $1/100$ & $4.6$ \\
            DQN-BBMDP & $100/100$ & $\textcolor{purple}{\mathbf{34/100}}$ & $\textcolor{purple}{\mathbf{2.3}}$ &  $\textcolor{blue}{\mathbf{100/100}}$ & $\textcolor{blue}{\mathbf{18/100}}$ & $\textcolor{blue}{\mathbf{2.9}}$ \\
            \bottomrule
            & \multicolumn{6}{c}{Combinatorial Auction}
        \hspace{1.5cm}
        \end{tabular}
        \begin{tabular}{ccccccc} 
            & \multicolumn{3}{c}{Train / Test} & \multicolumn{3}{c}{Transfer}\\
            Method & Solved & Wins & Rank & Solved &  Wins & Rank \\  
            \toprule
            SCIP  & $100/100$ & $9/100$ & $5.7$ & $100/100$ & $7/100$ & $4.5$ \\
            \midrule
             IL & $100/100$ & $\textcolor{purple}{\mathbf{72/100}}$ & $\textcolor{purple}{\mathbf{1.6}}$ & $100/100$ & $\textcolor{purple}{\mathbf{57/100}}$ & $\textcolor{purple}{\mathbf{1.7}}$ \\
             IL-DFS & $100/100$ & $10/100$ & $2.4$ & $100/100$ & $0/100$ & $3.2$\\
            \midrule
            PG-tMDP  & $100/100$ & $0/100$ & $4.9$ & $\textcolor{purple}{\mathbf{100/100}}$ & $\textcolor{blue}{\mathbf{36/100}}$ & $\textcolor{blue}{\mathbf{1.8}}$ \\
            DQN-tMDP  & $100/100$ & $1/100$ & $4.8$ & $85/100$ & $0/100$ & $6.1$ \\ 
            DQN-Retro & $100/100$ & $\textcolor{blue}{\mathbf{6/100}}$ & $5.4$ & $22/100$ & $0/100$ & $6.7$\\
            DQN-BBMDP & $100/100$ & $2/100$ & $\textcolor{blue}{\mathbf{3.3}}$ & $95/100$ & $0/100$ & $4.2$\\
            \bottomrule
            & \multicolumn{6}{c}{Maximum Independent Set}
        \hspace{1.5cm}
        \end{tabular}
        \begin{tabular}{ccccccc} 
            & \multicolumn{3}{c}{Train / Test} & \multicolumn{3}{c}{Transfer}\\
            Method & Solved & Wins & Rank & Solved &  Wins & Rank   \\  
            \toprule
            SCIP  & $100/100$ & $\textcolor{purple}{\mathbf{88/100}}$ & $\textcolor{purple}{\mathbf{1.4}}$ & $\textcolor{purple}{\mathbf{100/100}}$ & $\textcolor{purple}{\mathbf{60/100}}$ & $\textcolor{purple}{\mathbf{1.9}}$ \\
            \midrule
             IL & $100/100$ & $1/100$ & $4.5$ & $100/100$ & $6/100$ & $3.4$\\
             IL-DFS & $100/100$ & $1/100$ & $5.8$ & $98/100$ & $0/100$ & $6.0$ \\
            \midrule
            PG-tMDP  & $100/100$ & $0/100$ & $6.0$ & $98/100$ & $5/100$ & $5.0$ \\
            DQN-tMDP  & $100/100$ & $1/100$ & $3.5$ & $99/100$ & $\textcolor{blue}{\mathbf{14/100}}$ & $\textcolor{blue}{\mathbf{3.5}}$ \\ 
            DQN-Retro & $100/100$ & $3/100$ & $3.5$ & $98/100$ & $9/100$ & $3.8$ \\
            DQN-BBMDP & $100/100$ & $\textcolor{blue}{\mathbf{6/100}}$ & $\textcolor{blue}{\mathbf{3.3}}$ & $\textcolor{blue}{\mathbf{100/100}}$ & $6/100$ & $4.3$ \\
            \bottomrule
            & \multicolumn{6}{c}{Multiple Knapsack}
        \hspace{1cm}
        \end{tabular}
        \label{tab:win_results}
    \end{table}

\label{app:results}
    \begin{table}[]
        \centering
        \caption{Computational performance comparison on four MILP benchmarks. Following prior works, we report geometrical mean over 100 instances, averaged over 5 seeds, as well as per-benchmark standard deviations.}
        \hspace{5mm}
        \small
        \begin{tabular}{ccccccc} 
            & \multicolumn{3}{c}{Train / Test} &  \multicolumn{3}{c}{Transfer}\\
            Method & Nodes & Time & Solved & Nodes & Time & Solved \\ 
            \toprule
            Random & $3289 \pm 4.2\%$ & $5.9 \pm 4.3\%$ & $100/100$ &  $270365 \pm 9.5\%$ & $811 \pm 7.9\%$ & $60/100$ \\
            SB  & $35.8 \pm 0.0\%$ & $12.93 \pm 0.0\%$ & $100/100$ &  $672.1 \pm 0.0\%$ & $398 \pm 0.2\%$ & $82/100$ \\
            SCIP  & $62.0 \pm 0.0\%$ & $2.27 \pm 0.0\%$ & $100/100$ & $3309 \pm 0.0\%$ & $48.4 \pm 0.1\%$ & $100/100$ \\
            \midrule
             IL & \textcolor{purple}{$\mathbf{133.8 \pm 1.0\%}$} & $0.90 \pm 4.8\%$ & $100/100$ & \textcolor{purple}{$\mathbf{2610 \pm 0.7\%}$} & $23.1 \pm 1.5\%$ & \textcolor{purple}{$\mathbf{100/100}$} \\
             IL-DFS & $136.4 \pm 1.8\%$ & \textcolor{purple}{$\mathbf{0.74 \pm 5.3\%}$} & $100/100$ & $3103 \pm 2.0\%$ & \textcolor{purple}{$\mathbf{22.5 \pm 3.1\%}$} & \textcolor{purple}{$\mathbf{100/100}$} \\
            \midrule
            PG-tMDP  & $649.4 \pm 0.7\%$ & $2.32 \pm 2.4\%$ & $100/100$ & $44649 \pm 3.7\%$ & $221 \pm 4.1\%$ & $78   /100$ \\
            DQN-tMDP  & $175.8\pm 1.1\%$ & $0.83 \pm 4.5\%$ & $100/100$ & $8632 \pm 4.9\%$ & $71.3 \pm 5.8\%$ & $96/100$ \\ 
            DQN-Retro & $183.0 \pm 1.2\%$ & $1.14 \pm 4.1\%$ & $100/100$ & $6100 \pm 4.2\%$ & $59.4 \pm 4.2\%$ & $98/100$ \\
            DQN-BBMDP & \textcolor{blue}{$\mathbf{152.3 \pm 0.6\%}$} & \textcolor{blue}{$\mathbf{0.77 \pm 5.6\%}$} & $100/100$ & \textcolor{blue}{$\mathbf{5651 \pm 2.2\%}$} & \textcolor{blue}{$\mathbf{46.4 \pm 3.3\%}$} & \textcolor{blue}{$\mathbf{100/100}$} \\
            \bottomrule
            & \multicolumn{6}{c}{Set covering}
        \hspace{1cm}
        \end{tabular}
        \begin{tabular}{ccccccc} 
            & \multicolumn{3}{c}{Train / Test} & \multicolumn{3}{c}{Transfer}\\
            Method & Nodes & Time & Solved & Nodes & Time & Solved \\  
            \toprule
            Random & $1111 \pm 4.3\%$ & $2.16 \pm 6.6\%$ & $100/100$ & $354650 \pm 6.7\%$ & $814 \pm 7.1\%$ & $64/100$ \\
            SB  & $28.2 \pm 0.0\%$ & $6.21 \pm 0.1\%$ & $100/100$ & $389.6 \pm 0.0\%$ & $255 \pm 0.2\%$ & $88/100$ \\
            SCIP  & $20.2 \pm 0.0\%$ & $1.77 \pm 0.1\%$ & $100/100$ & $1376 \pm 0.0\%$ & $14.
            77\pm 0.1\%$ & $100/100$ \\
            \midrule
             IL & \textcolor{purple}{$\mathbf{83.6 \pm 0.8\%}$} & $0.65 \pm 7.3\%$ & $100/100$ & \textcolor{purple}{$\mathbf{1309 \pm 1.6\%}$} & \textcolor{purple}{$\mathbf{9.4 \pm 2.2\%}$} & \textcolor{purple}{$\mathbf{100/100}$} \\
             IL-DFS & $95.5 \pm 0.9\%$ & \textcolor{purple}{$\mathbf{0.56 \pm 7.1\%}$} & $100/100$ & $1802 \pm 2.0\%$ & $10.2 \pm 1.8\%$ & $100/100$ \\
            \midrule
            PG-tMDP  & $168.0 \pm 2.8\%$ & $0.94 \pm 6.0\%$ & $100/100$ & $6001 \pm 2.7\%$ & $30.7 \pm 2.4\%$ & $100/100$ \\
            DQN-tMDP  & $203.3 \pm 4.2\%$ & $1.11 \pm 4.0\%$ & $100/100$ & $20553 \pm 3.8\%$ & $116 \pm 3.9\%$ & $100/100$ \\ 
            DQN-Retro & $103.2 \pm 1.2\%$ & $0.78 \pm 7.5\%$ & $100/100$ & $2908 \pm 1.7\%$ & $18.4 \pm 2.7\%$ & $100/100$ \\
            DQN-BBMDP & \textcolor{blue}{$\mathbf{97.9 \pm 1.2\%}$} & \textcolor{blue}{$\mathbf{0.62 \pm 8.5\%}$} & $100/100$ & \textcolor{blue}{$\mathbf{2273 \pm 1.9\%}$} & \textcolor{blue}{$\mathbf{11.8 \pm 2.0\%}$} & \textcolor{blue}{$\mathbf{100/100}$} \\
            \bottomrule
            & \multicolumn{6}{c}{Combinatorial auction}
        \hspace{5mm}
        \end{tabular}
                \begin{tabular}{ccccccc} 
            & \multicolumn{3}{c}{Train / Test} & \multicolumn{3}{c}{Transfer}\\
            Method & Nodes & Time & Solved & Nodes & Time & Solved \\ 
            \toprule
            Random & $386.8 \pm 5.4\%$ & $2.01 \pm 4.8\%$ & $100/100$ & $215879 \pm 6.7\%$ & $2102 \pm 6.2\%$ & $25/100$ \\
            SB  & $24.9 \pm 0.0\%$ & $45.87 \pm 0.4\%$ & $100/100$ & $169.9 \pm 0.2\%$ & $2172 \pm 0.9\%$ & $15/100$ \\
            SCIP  & $19.5 \pm 0.0\%$ & $2.44 \pm 0.4\%$ & $100/100$ & $3368 \pm 0.0\%$ & $90.0 \pm 0.2\%$ & $100/100$ \\
            \midrule
             IL & \textcolor{purple}{$\mathbf{40.1 \pm 3.45\%}$} & \textcolor{purple}{$\mathbf{0.36 \pm 3.1\%}$} & $100/100$ & \textcolor{purple}{$\mathbf{1882 \pm 4.0\%}$} & \textcolor{purple}{$\mathbf{38.6 \pm 3.2\%}$} & \textcolor{purple}{$\mathbf{100/100}$} \\
             IL-DFS & $69.4 \pm 6.5\%$ & $0.44 \pm 4.8\%$ & $100/100$ & $3501 \pm 2.7\%$ & $51.9 \pm 2.6\%$ & $100/100$ \\
            \midrule
            PG-tMDP  & $153.6 \pm 5.0\%$ & $0.92 \pm 2.6\%$ & $100/100$ & \textcolor{blue}{$\mathbf{3133 \pm 4.6\%}$} & \textcolor{blue}{$\mathbf{39.5 \pm 3.8\%}$} & \textcolor{blue}{$\mathbf{100/100}$} \\
            DQN-tMDP  & $168.0 \pm 5.6\%$ & $1.00 \pm 3.4\%$ & $100/100$ & $45634 \pm 7.4\%$ & $477 \pm 5.1\%$ & $85/100$ \\ 
            DQN-Retro & $223.0 \pm 4.1\%$ & $1.81 \pm 3.6\%$ & $100/100$ & $119478 \pm 6.1\%$ & $1863 \pm 4.8\%$ & $22/100$ \\
            DQN-BBMDP & \textcolor{blue}{$\mathbf{103.2 \pm 9.3\%}$} & \textcolor{blue}{$\mathbf{0.62 \pm 6.8\%}$} & $100/100$ & $7168 \pm 5.3\%$ & $81.3 \pm 4.2\%$ & $95/100$ \\
            \bottomrule
            & \multicolumn{6}{c}{Maximum independent set}
        \hspace{1cm}
        \end{tabular}
        \begin{tabular}{ccccccc} 
            & \multicolumn{3}{c}{Train / Test} & \multicolumn{3}{c}{Transfer}\\
            Method & Nodes & Time & Solved & Nodes & Time & Solved \\ 
            \toprule
            Random & $733.5 \pm 13.0\%$ & $0.55 \pm 6.9\%$ & $100/100$ & $93452\pm 14.3\%$ & $70.6 \pm 9.2\%$ & $99/100$ \\
            SB  & $161.7 \pm 0.0\%$ & $0.69\pm 0.1\%$ & $100/100$ & \textcolor{purple}{$\mathbf{1709 \pm 0.5\%}$} & \textcolor{purple}{$\mathbf{12.5 \pm 0.9\%}$} & \textcolor{purple}{$\mathbf{100/100}$} \\
            SCIP  & $289.5 \pm 0.0\%$ & $0.53 \pm 0.2\%$ & $100/100$ & $30260 \pm 0.0\%$ & $22.14 \pm 0.2\%$ & $100/100$ \\
            \midrule
             IL & $272.0 \pm 12.9\%$ & $0.69 \pm 8.5\%$ & $100/100$ & $9747 \pm 7.5\%$ & $46.5 \pm 6.6\%$ & $100/100$ \\
             IL-DFS & $472.8 \pm 13.0\%$ & $1.07 \pm 9.0\%$ & $100/100$ & $43224 \pm 9.0\%$ & $131 \pm 8.6\%$ & $98/100$ \\
            \midrule
            PG-tMDP  & $436.9 \pm 21.2\%$ & $1.57 \pm 16.9\%$ & $100/100$ & $35614 \pm 14.3\%$ & $123 \pm 15.4\%$ & $98/100$ \\
            DQN-tMDP  & $266.4 \pm 7.2\%$ & $0.73 \pm 4.6\%$ & $100/100$ & \textcolor{blue}{$\mathbf{22631 \pm 8.6\%}$} & \textcolor{blue}{$\mathbf{65.1 \pm 5.5\%}$} & $99/100$ \\ 
            DQN-Retro & $250.3 \pm 9.5\%$ & $0.67 \pm 5.0\%$ & $100/100$ & $27077 \pm 8.8\%$ & $79.5 \pm 6.2\%$ & $98/100$ \\
            DQN-BBMDP & \textcolor{purple}{$\mathbf{236.6 \pm 6.4\%}$} & \textcolor{purple}{$\mathbf{0.66 \pm 2.7\%}$} & $100/100$ & $37098 \pm 7.0\%$ & $109 \pm 4.9\%$ & \textcolor{blue}{$\mathbf{100/100}$} \\
            \bottomrule
            & \multicolumn{6}{c}{Multiple knapsack}
        \hspace{5mm}
        \end{tabular}
        \label{tab:complete_results}
    \end{table}

\end{document}